\documentclass[10pt]{article}
\usepackage{amsfonts,amsmath,amssymb,amsthm,mathtools,bbm,bm}
\usepackage{dsfont}
\usepackage[usenames, dvipsnames]{color}
\usepackage{graphicx}
\usepackage{hyperref}
\usepackage{enumitem}
\usepackage[numbers]{natbib}
\usepackage{authblk}
\usepackage{mathpazo}
\usepackage{commath}
\usepackage{geometry}

\newtheorem{Theorem}{Theorem}[section]
\newtheorem{Lemma}[Theorem]{Lemma}
\newtheorem{Proposition}[Theorem]{Proposition}

\theoremstyle{definition}
\theoremstyle{definition}
\newtheorem{Definition}[Theorem]{Definition}

\theoremstyle{remark}

\numberwithin{equation}{section}
\newenvironment{equations}{\equation\aligned}{\endaligned\endequation}
\newenvironment{equations*}{\equation\aligned\nonumber}{\endaligned\endequation}

\hypersetup{colorlinks=True, citecolor=BurntOrange, linkcolor=RedViolet, urlcolor=Blue}

\DeclareMathOperator*{\argmin}{\arg\!\min}

\newcommand{\ceil}[1]{\left\lceil {#1} \right\rceil}
\newcommand{\floor}[1]{\left\lfloor {#1} \right\rfloor}

\newcommand{\indi}[1]{\mathds{1}\left( {#1} \right)}

\def\one{\mathbf 1}

\def\d{\textup{d}}
\def\e{\textup{e}}

\def\N{\mathsf{N}}
\def\E{\mathsf{E}}
\def\P{\mathsf{P}}
\def\R{\mathbb{R}}

\makeatletter
\def\greekvectors#1{%
 \@for\next:=#1\do{%
    \def\X##1;{\expandafter\def\csname b##1\endcsname{\bm{\csname##1\endcsname}}}
    \expandafter\X\next;}
 \@for\next:=#1\do{%
    \def\X##1;{\expandafter\def\csname h##1\endcsname{\widehat{\csname##1\endcsname}}}
    \expandafter\X\next;}
 \@for\next:=#1\do{%
    \def\X##1;{\expandafter\def\csname c##1\endcsname{\check{\csname##1\endcsname}}}
    \expandafter\X\next;}
 \@for\next:=#1\do{%
    \def\X##1;{\expandafter\def\csname hb##1\endcsname{\widehat{\bm{\csname##1\endcsname}}}}
    \expandafter\X\next;}
}
\greekvectors{alpha,beta,gamma,delta,epsilon,zeta,theta,kappa,lambda,
    mu,nu,xi,pi,rho,tau,phi,chi,psi,omega,
    Delta,Gamma,Theta,Lambda,Xi,Pi,Sigma,Phi,Psi,Omega
}
\@tfor\next:=abcfghijklmnopqrstuvwxyzABCDGHIJKLMOQSTUVWXYZ\do{%
    \def\command@factory#1{\expandafter\def\csname #1\endcsname{\mathbf{#1}} }
    \expandafter\command@factory\next
}
\@tfor\next:=abcdefghijklmnpqrstuvwxyzABCDEFGHIJKLMNOPQRSTUVWXYZ\do{%
    \def\command@factory#1{\expandafter\def\csname t#1\endcsname{\widetilde{#1}} }
    \expandafter\command@factory\next
}
\@tfor\next:=abcdefghijklmnopqrstuvwxyzABCDEFGHIJKLMNOPQRSTUVWXYZ\do{%
    \def\command@factory#1{\expandafter\def\csname tb#1\endcsname{\tilde{\mathbf{#1}}} }
    \expandafter\command@factory\next
}
\@tfor\next:=abcdefghijklmnopqrstuvwxyzABCDEFGHIJKLMNOPQRSTUVWXYZ\do{%
    \def\command@factory#1{\expandafter\def\csname hb#1\endcsname{\widehat{\mathbf{#1}}} }
    \expandafter\command@factory\next
}
\@tfor\next:=ABCDEFGHIJKLMNOPQRSTUVWXYZ\do{%
    \def\command@factory#1{\expandafter\def\csname b#1\endcsname{\mathbbm{#1}} }
    \expandafter\command@factory\next
}
\@tfor\next:=ABCDEFGHIJKLMNOPQRSTUVWXYZ\do{%
    \def\command@factory#1{\expandafter\def\csname c#1\endcsname{\mathcal{#1}} }
    \expandafter\command@factory\next
}
\@tfor\next:=abcdefghjklmnopqrstuvwxyzABCDEFGHIJKLMNOPQRSTUVWXYZ\do{
    \def\command@factory#1{\expandafter\def\csname f#1\endcsname{\mathfrak{#1}} }
    \expandafter\command@factory\next
}
\@tfor\next:=ABCDEFGHIJKLMNOPQRSTUVWXYZ\do{%
    \def\command@factory#1{\expandafter\def\csname s#1\endcsname{\mathsf{#1}} }
    \expandafter\command@factory\next
}

\begin{document}

\title{\textsc{Smooth function approximation by deep neural networks with general activation functions}\footnote{This work was supported by Samsung Science and Technology Foundation under Project Number SSTF-BA1601-02.}}
\author{Ilsang \textsc{Ohn} and Yongdai \textsc{Kim}\footnote{Corresponding author. E-mail: ydkim0903@gmail.com}}
\affil{\textit{Department of Statistics, Seoul National University}}
\maketitle

\begin{abstract}
There has been a growing interest in expressivity of deep neural networks. However, most of the existing work about this topic focuses only on the specific activation function such as ReLU or sigmoid. In this paper, we investigate the approximation ability of deep neural networks with a broad class of activation functions. This class of activation functions includes most of frequently used activation functions. We derive the required depth, width and sparsity of a deep neural network to approximate any H\"older smooth function upto a given approximation error for the large class of activation functions.  Based on our approximation error analysis, we derive the minimax optimality of the deep neural network estimators with the general activation functions in both regression and classification problems.

\null\noindent
\end{abstract}

\section{Introduction}

Neural networks are learning machines motivated by the architecture of the human brain. Neural networks are comprised of multiple hidden layers, and each of the hidden layers has multiple hidden nodes which consist of an affine map of the outputs from the previous layer and a nonlinear map called an {activation function}. Deep neural networks have been leading tremendous success in various pattern recognition and machine learning tasks such as object recognition, image segmentation, machine translation and others. For an overview on the empirical success of deep neural networks, we refer to the review paper \citep{lecun2015deep} and recent book \citep{goodfellow2016deep}.

Inspired by the success of deep neural networks, many researchers have tried to give theoretical supports for the success of deep neural networks. Much of the work upto date has focused on the expressivity of deep neural networks, i.e., ability to approximate a rich class of functions efficiently.
The well-known classical result on this topic is the universal approximation theorem, which states that every continuous function can be approximated arbitrarily well by a neural network \citep{cybenko1989approximation, hornik1989multilayer, funahashi1989approximate, chui1992approximation, leshno1993multilayer}. But these results do not specify the required numbers of layers and nodes of a neural network to achieve a given approximation~accuracy. 

Recently, several results about the effects of the numbers of layers and nodes of a deep neural network to its expressivity have been reported. They provide upper bounds of the numbers of layers and nodes required for neural networks to uniformly approximate all functions of interest. Examples of a class of functions include the space of rational functions of polynomials \citep{telgarsky2017neural}, the H\"older space~\citep{yarotsky2017error, schmidt2017nonparametric, bauer2019deep, li2019better}, Besov and mixed Besov spaces~\citep{suzuki2018adaptivity} and even a class of discontinuous functions \citep{petersen2018optimal, imaizumi2018deep}. 

The nonlinear activation function is a central part that makes neural networks differ from the linear models, that is, a neural network becomes a linear function if the linear activation function is used. Therefore, the choice of an activation function substantially influences on the performance and computational efficiency. Numerous activation functions have been suggested to improve neural network learning \citep{bergstra2009quadratic, clevert2015fast, carlile2017improving,  ramachandran2017searching, klimek2018neural, wuraola2018sqnl}. We refer to the papers \citep{glorot2011deep, ramachandran2017searching} for an overview of this topic.

There are also many recent theoretical studies about the approximation ability of deep neural networks. However, most of the studies focus on a specific type of the activation function such as ReLU \citep{yarotsky2017error, schmidt2017nonparametric, petersen2018optimal, imaizumi2018deep, suzuki2018adaptivity}, or small classes of activation functions such as sigmoidal functions with additional monotonicity, continuity, and/or boundedness conditions \citep{mhaskar1993approximation, costarelli2017saturation, costarelli2018solving, costarelli2018estimates, costarelli2018approximation} and $m$-admissible functions which are sufficiently smooth and bounded \citep{bauer2019deep}. For definitions of sigmoidal and $m$-admissible functions, see \citep{costarelli2017saturation} and \citep{bauer2019deep}, respectively. Thus a unified theoretical framework still lacks. 

In this paper, we investigate the approximation ability of deep neural networks with a quite general class of activation functions. We derive the required numbers of layers and nodes of a deep neural network to approximate  any H\"older smooth function upto a given approximation error for the large class of activation functions. Our specified class of activation functions and the corresponding approximation ability of deep neural networks include most of previous results~\citep{yarotsky2017error, schmidt2017nonparametric, mhaskar1993approximation, bauer2019deep} as special~cases.

Our general theoretical results of the approximation ability of deep neural networks enables us to study statistical properties of deep neural networks. \citet{schmidt2017nonparametric} and Kim et al.~\cite{kim2018fast} proved the minimax optimality of a deep neural network estimator with the ReLU activation function in regression and classification problems, respectively. In this paper, we derive similar results for general activation functions.

This paper is structured as follows. In Section \ref{sec:dnn}, we introduce some notions about deep neural networks. In Section \ref{sec:act}, we introduce two large classes of activation functions. In Section \ref{sec:approx}, we present our main result on the approximation ability of a deep neural network with the general activation function considered in Section \ref{sec:act}. In Section \ref{sec:slt}, we apply the result in Section \ref{sec:approx} to the supervised learning problems of regression and classification. Conclusions are given in Section \ref{sec:con}. The proofs of all  results are given in Appendix.

\subsection*{Notation}

We denote by $\indi{\cdot}$ the indicator function. Let $\R$ be the set of real numbers and $\bN$ be the set of natural numbers. For a real valued vector $\x\equiv(x_1,\dots, x_d)$, we let $\abs{\x}_0:=\sum_{j=1}^d\mathds{1}(x_j\neq 0)$, $\abs{\x}_p:=\del{\sum_{j=1}^d|x_j|^p}^{1/p}$ for $p\in[1,\infty)$ and $\abs{\x}_\infty:=\max_{1\le j\le d}|x_j|$. For simplicity, we let $|\x|:=|\x|_1.$ For a real valued function $f(x):\R\to \R$, we let $f'(a), f''(a)$ and $f'''(a)$ are the first, second and third order derivatives of $f$ at $a$, respectively. We let $f'(a+):=\lim_{\epsilon\downarrow0}(f(a+\epsilon)-f(a))/\epsilon$ and $f'(a-):=\lim_{\epsilon\downarrow0}(f(a-\epsilon)-f(a))/\epsilon$.  For $x\in\R$, we write $(x)_+:=\max\{x,0\}$.

\section{Deep Neural Networks}
\label{sec:dnn}

In this section we provide a mathematical representation of deep neural networks. A neural network with $L\in\bN$ layers, $n_l\in\bN$ many nodes at the $l$-th hidden layer for $l=1,\dots, L$, input of dimension $n_0$, output of dimension $n_{L+1}$ and nonlinear activation function $\sigma:\R\to\R$ is expressed as
    \begin{equation}
    \label{eq:nn}
        N_\sigma(\x|\btheta):=\sA_{L+1}\circ\sigma_L\circ\sA_{L}\circ\cdots \circ\sigma_1\circ\sA_1(\x),
    \end{equation}
where $\sA_l:\R^{n_{l-1}}\to \R^{n_l}$ is an affine linear map defined by $\sA_l(\x)=\W_l\x+\b_l$ for given  $n_l\times n_{l-1}$ dimensional weight matrix $\W_l$ and $n_l$ dimensional bias vector $\b_l$ and $\sigma_l:\R^{n_l}\to\R^{n_l}$ is an element-wise nonlinear activation map defined by $\sigma_l(\z):=(\sigma(z_1),\dots, \sigma(z_{n_l}))^\top$. Here, $\btheta$ denotes the set of all weight matrices and bias vectors  $\btheta:=\del[1]{(\W_1,\b_1),(\W_2,\b_2),\dots, (\W_{L+1}, \b_{L+1})}, $ 
which we call $\btheta$ the parameter of the neural network, or simply, a {network parameter}.
 
We introduce some notations related to the network parameter. For a network parameter $\btheta$, we write $L(\btheta)$ for the number of hidden layers of the corresponding neural network, and write $n_{\max}(\btheta)$ for the maximum of the numbers of hidden nodes at each layer. Following a standard convention, we say that $L(\btheta)$ is the {depth} of the deep neural network and $n_{\max}(\btheta)$ is the {width} of the deep neural network. We let $\abs{\btheta}_0$ be the number of nonzero elements of $\btheta$, i.e.,
      \begin{equation*}
         \abs{\btheta}_0:=\sum_{l=1}^{L+1}\left( \abs{\text{vec}(\W_l)}_0 +\abs{\b_l}_0\right),
      \end{equation*}
where $\text{vec}(\W_l)$ transforms the matrix $\W_l$ into the corresponding vector by concatenating the column vectors. We call $\abs{\btheta}_0$ {sparsity} of the deep neural network. Let $\abs{\btheta}_\infty$ be the largest absolute value of elements of $\btheta$, i.e., 
    \begin{equation*}
        \abs{\btheta}_\infty :=\max \left\{ \max_{1\le l\le L+1} \abs{\text{vec}(\W_l)}_\infty, 
\max_{1\le l\le L+1} \abs{\b_l}_\infty\right\}.
    \end{equation*}
We call $\abs{\btheta}_\infty$ {magnitude} of the deep neural network. We let $\textsf{in}(\btheta)$ and $\textsf{out}(\btheta)$ be the input and output dimensions of the deep neural network, respectively. We denote by $\Theta_{d,o}(L,N)$ the set of network parameters with depth $L$, width $N$, input dimension $d$ and output dimension $o$,  that is,
    \begin{equation*}
        \Theta_{d,o}(L, N):=\cbr{\btheta:L(\btheta)\le L, n_{\max}(\btheta)\le N, \textsf{in}(\btheta)=d, \textsf{out}(\btheta)=o}.
    \end{equation*}    
We further define a subset of $\Theta_{d,o}(L,N)$ with restrictions on sparsity and magnitude as
    \begin{equation*}
        \Theta_{d,o}(L, N, S, B):=\cbr{\btheta\in\Theta_{d,o}(L,N): \abs{\btheta}_0\le S,\abs{\btheta}_\infty \le B}.
    \end{equation*}

\section{Classes of Activation Functions}
\label{sec:act}

In this section, we consider two classes of activation functions. These two classes include most of commonly used activation functions. Definitions and examples of each class of activation functions  are provided in the consecutive two subsections.

\subsection{Piecewise Linear Activation Functions}

We first consider piecewise linear activation functions.

\begin{Definition}
A  function $\sigma:\R\to\R$ is continuous piecewise linear if it is continuous and there exist a finite number of break points $a_1\le a_2\le \cdots\le  a_K\in\R$ with $K\in\bN$ such that  $\sigma'(a_k-)\neq \sigma'(a_k+)$ for every $k=1,\dots, K$ and $\sigma(x)$ is linear on $(-\infty,a_1], [a_1,a_2],\dots, [a_{K-1}, a_K], [a_K, \infty)$.
\end{Definition}

Throughout this paper, we write ``picewise linear'' instead of ``continuous picewise linear'' for notational simplicity unless there is a confusion. The representative examples of piecewise linear activation functions are as follows:
    \begin{itemize}
        \item ReLU: $\sigma(x)=\max\{x,0\}$.
        \item Leaky ReLU: : $\sigma(x)=\max\{x,ax\}$ for $a\in(0,1)$.
    \end{itemize}

The ReLU activation function is the most popular choice in practical applications due to better gradient propagation and efficient computation \citep{glorot2011deep}. In this reason, most of the recent results on the function approximation by deep neural networks are based on the ReLU activation \mbox{function \citep{yarotsky2017error, schmidt2017nonparametric, petersen2018optimal, imaizumi2018deep, suzuki2018adaptivity}}. In Section \ref{sec:approx}, as \citet{yarotsky2017error} did, we extend these results to any continuous piecewise linear activation function by showing that the ReLU activation function can be exactly represented by a linear combination of piecewise linear activation functions. A formal proof for this argument is presented in Appendix \ref{sec:proof1}.

\subsection{Locally Quadratic Activation Functions}

One of the basic building blocks in approximation by deep neural networks is the square function, which should be approximated precisely. Piecewise linear activation functions have zero curvature (i.e., constant first-order derivative) inside each interval divided by its break points, which makes it relatively difficult to approximate the square function efficiently. But if there is an interval on which the activation function has nonzero curvature, the square function can be approximated more efficiently, which is a main motivation of considering a new class of activation functions that we call locally quadratic.

\begin{Definition}
A function $\sigma:\R\to\R$ is locally quadratic if there exits an open interval $(a,b)\subset\R$ on which $\sigma$ is three times continuously differentiable with bounded derivatives and there exists $t\in(a,b)$ such that $\sigma'(t)\neq0$ and $\sigma''(t)\neq0$. 
\end{Definition}

We now give examples of locally quadratic activation functions. First of all, any nonlinear smooth activation function with nonzero second derivative, is locally quadratic. Examples are:

    \begin{itemize}
        \item Sigmoid: $\displaystyle \sigma (x)=\frac{1}{1+\e^{-x}}.$   \medskip
        \item Tangent hyperbolic: $\displaystyle \sigma(x)=\frac{\e^x-\e^{-x}}{\e^x+\e^{-x}}.$  \medskip
        \item Inverse square root unit (ISRU) \citep{carlile2017improving}: 
            $\displaystyle  \sigma(x)=\frac{x}{\sqrt{1+ax^2}}$ for $a>0$. \medskip
        \item Soft clipping \citep{klimek2018neural}: 
            $\displaystyle \sigma(x)=\frac{1}{a}\log \del{\frac{1+\e^{ax}}{1+\e^{a(x-1)}}}$ for $a>0$.  \medskip
        \item SoftPlus \citep{glorot2011deep}: $\sigma(x)=\log (1+\e^x)$.  \medskip
         \item Swish \citep{ramachandran2017searching}: $\displaystyle \sigma(x)=\frac{x}{1+\e^{-x}}$.
    \end{itemize}
    
In addition, piecewise smooth function having nonzero second derivative on at least one piece, is also locally quadratic. Examples are:
    \begin{itemize}
        \item Rectified power unit (RePU) \citep{li2019better}: 
            $\displaystyle \sigma(x)=\max\{x^k,0\}$ for $k\in \bN\setminus\{1\}$. \medskip
        \item Exponential linear unit (ELU) \citep{clevert2015fast}: 
            $\displaystyle \sigma(x)=a(\e^x-1)\indi{x\le0}+ x\indi{x>0}$ for $a>0$. \medskip
		\item Inverse square root linear unit  (ISRLU) \citep{carlile2017improving}: 
            $\displaystyle  \sigma(x)=\frac{x}{\sqrt{1+ax^2}}\indi{x\le0}+x\indi{x>0}$ for $a>0$. \medskip
        \item Softsign \citep{bergstra2009quadratic}: 
            $\displaystyle \sigma(x)=\frac{x}{1+|x|}.$  \medskip
        \item Square nonlinearity \citep{wuraola2018sqnl}:\\
         $\displaystyle \sigma(x)=\indi{x>2}+(x-x^2/4)\indi{0\le x\le2}+(x+x^2/4)\indi{-2\le x<0} -\indi{x<-2}$.
    \end{itemize}

\section{Approximation of Smooth Functions by Deep Neural Networks}
\label{sec:approx}

In this section we introduce the function class we consider and show the approximation ability of the deep neural networks with a activation function considered in Section \ref{sec:act}.

\subsection{H\"older Smooth Functions}

We recall the definition of H\"older smooth functions. For a $d$-dimensional multiple index \mbox{$\m\equiv(m_1,\dots, m_d)\in \bN_0^d$} where $\bN_0:=\bN\cup\{0\}$, we let $\x^{\m}:=x_1^{m_1}\cdots x_d^{m_d}$ for $\x\in\R^d$. 
For a function $f:\cX\to\R$, where $\cX$ denotes the domain of the function, we let $\|f\|_\infty:=\sup_{\x\in\cX}|f(\x)|$. We use~notation
    \begin{equation*}
        \partial^{\m}f:=\frac{\partial^{|\m|}f}{\partial \x^{\m}}=\frac{\partial^{|\m|}f}{\partial x_1^{m_1} \cdots\partial x_d^{m_d}},
    \end{equation*}
for $\m\in \bN_0^d$ to denote a derivative of $f$ of order $\m$. We denote by $\cC^m(\cX)$, the space of $m$ times differentiable functions on $\cX$ whose partial derivatives of order $\m$ with $|\m|\le m$ are continuous. We define the H\"older coefficient of order  $s\in(0,1]$ as
    \begin{equation*}
        [f]_{s}:=\sup_{\x_1,\x_2\in \cX, \x_1\neq\x_2 }\frac{|f(\x_1)-f(\x_2)|}{|\x_1-\x_2|^s}.
    \end{equation*}
    
For a positive real value $\alpha$, the H\"older space of order $\alpha$ is defined as 
	\begin{equation*}
	\cH^{\alpha}(\cX):=\left\{f\in \cC^{\floor{\alpha}}(\cX):\|f\|_{\cH^{\alpha}(\cX)}< \infty\right\},
	\end{equation*}
where $\|f\|_{\cH^{\alpha}(\cX)}$ denotes the H\"older norm defined by
    \begin{equation*}
        \|f\|_{\cH^{\alpha}(\cX)}
        :=\sum_{\m\in\bN_0^d:|\m|\le \floor{\alpha}}\|\partial^{\m}f\|_{\infty}
        +\sum_{\m\in\bN_0^d:|\m|= \floor{\alpha}}[\partial^{\m}f]_{\alpha- \floor{\alpha}}.
    \end{equation*}
    
We denote by $\cH^{\alpha, R}(\cX)$ the closed ball in the H\"older space of radius $R$ with respect to the H\"older norm, i.e.,
    \begin{equation*}
        \cH^{\alpha, R}(\cX)
        :=\left\{f\in\cH^{\alpha}(\cX): \|f\|_{\cH^{\alpha}(\cX)}\le R\right\}.
    \end{equation*}

\subsection{Approximation of H\"older Smooth Functions}

We present our main theorem in this section. 

\begin{Theorem}
\label{thm:main1}
Let $d\in\bN$, $\alpha>0$ and $R>0$. Let the activation function $\sigma$ be either continuous piecewise linear or locally quadratic. Let $f\in\cH^{\alpha, R}([0,1]^d)$. Then there exist positive constants $L_0$, $N_0$, $S_0$ and $B_0$ depending only on $d$, $\alpha$, $R$ and $\sigma$  such that, for any $\epsilon>0$, there is a neural network 
    \begin{equation}
        \btheta_\epsilon\in  \Theta_{d,1}\del{L_0\log(1/\epsilon),N_0\epsilon^{-d/\alpha},S_0\epsilon^{-d/\alpha}\log(1/\epsilon), B_0\epsilon^{-4(d/\alpha+1)}}
    \end{equation}
satisfying
    \begin{equation}
        \sup_{\x\in[0,1]^d}\abs{f(\x)-N_\sigma(\x|\btheta_\epsilon)}\le \epsilon.
    \end{equation}
\end{Theorem}

The result of Theorem \ref{thm:main1} is equivalent to the results on the approximation by ReLU neural \mbox{networks~\citep{yarotsky2017error, schmidt2017nonparametric}} in a sense that the upper bounds of the depth, width and sparsity are the same orders of those for ReLU, namely, depth $= O(\log(1/\epsilon))$, width $= O(\epsilon^{-d/\alpha})$ and sparsity $=O(\epsilon^{-d/\alpha}\log (1/\epsilon))$.  We remark that each upper bound is equivalent to the corresponding lower bound established by \citep{yarotsky2017error} up to logarithmic factor. 

For piecewise linear activation functions,  \citet{yarotsky2017error} derived similar results to ours. For locally quadratic activation functions, only special classes of activation functions were considered in the previous work. Li et al. \cite{li2019better} considered the RePU activation function and \citet{bauer2019deep} considered sufficiently smooth and bounded activation functions which include the sigmoid, tangent hyperbolic, ISRU and soft clipping activation functions. 
However, soft plus, swish, ELU, ISRLU, softsign and square nonlinearity activation functions are new ones only considered in our results.

Even if the orders of the depth, width and sparsity are the same for both both piecewise linear and locally quadratic activation functions, the ways of approximating a smooth function by use of these two activation function classes are quite different. To describe this point, let us provide an outline of the proof. We first consider equally spaced grid points with length $1/M$ inside the $d$-dimensional unit hypercube $[0,1]^d$. Let $\bG_{d,M}$ be the set of such grid points, namely,
    \begin{equation*}
        \bG_{d,M}:=\cbr{\frac{1}{M}(m_1,\dots, m_d):m_j\in \{0,1,\dots, M\},j=1,\dots, d}.
    \end{equation*}
    
For a given H\"older smooth function $f$ of order $\alpha$, we first find a ``local'' function for each grid that approximates the target function near the grid point but vanishes at apart from the grid point. To be more specific, we construct the local functions  $g_\z$, $\z\in\bG_{d, M}$ which satisfies:
    \begin{equation}
    \label{eq:local}
        \sup_{x\in [0,1]^d}\abs{f(\x)-\sum_{\z\in\bG_{d, M}}g_{\z,M}(\x)}\le C |\bG_{d,M}|^{-\alpha/d},
    \end{equation}
for some universal constant $C>0$. The inequality (\ref{eq:local}) implies that the more grid points we used, the more accurate approximation we get. Moreover, the quality of approximation is improved when the target function is more smooth (i.e., large $\alpha$) and low dimensional (i.e., small $d$ ). In fact, $g_{\z,M}(\x)$ is given by a product of the Taylor polynomial $P_{\z, M}(\x):= \sum_{\m\in \bN_0^d:|\m|\le \alpha} \del{\partial^{\m}f}(\z)\frac{(\x-\z)^{\m}}{\m!}$ at $\z$ and the local basis function  $\phi_{\z,M}(\x):=\prod_{j=1}^d(1/M-|x_j-z_j|)_+$, where $\m!:=\prod_{j=1}^d m_j!$. By simple algebra, we~have
    \begin{align*}
        P_M(\x):=\sum_{\z\in\bG_{d, M}}g_{\z,M}(\x)&:=\sum_{\z\in\bG_{d, M}}P_{ \z, M}(\x)\phi_{\z,M}(\x)\\
       &= \sum_{\z\in\bG_{d, M}}\sum_{\m:|\m|\le \alpha}\beta_{\z,\m}\x^{\m}\phi_{\z,M}(\x),
    \end{align*}
where $\beta_{\z,\m}:=\sum_{\tilde\m:\tilde\m\ge\m, |\tilde\m|\le\alpha}
    \del{\partial^{\tilde\m}f}(\z)\frac{(-\z)^{\tilde\m-\m}}{\m!(\tilde\m-\m)!}$.

The second stage is to approximate each monomial $\x^\m$ and each local basis function $\phi_{\z,M}(\x)$ by deep neural networks. Each monomial can be approximated more efficiently by a deep neural network with a locally quadratic activation function than a piecewise linear activation function since each monomial has nonzero curvature. On the other hand, the local basis function can be approximated more efficiently by a deep neural network with a piecewise linear activation than a locally quadratic activation function since the local basis function is piecewise linear itself. That is, there is a trade-off in using either a piecewise linear or a locally quadratic activation function.

We close this section by giving a comparison of our result to the approximation error analysis of~\citep{bauer2019deep}. \citet{bauer2019deep} studies approximation of the H\"older smooth function of order $\alpha$ by a two layer neural network with $m$-admissible activation functions with $m\ge \alpha$, where a function $\sigma$ is called $m$-admissible if (1) $\sigma$ is at least $m+1$ times continuously differentiable with bounded derivatives; (2) a point $t\in\R $ exists, where all derivatives up to the order $m$ of $\sigma$ are different from zero; and (3) $|\sigma(x)-1|\le 1/x$ for $x > 0$ and $|\sigma(x)|\le 1/|x|$ for $x<0$. Our notion of locally quadratic activation functions is a generalized version of the $m$-admissibility.

In the proof of \citep{bauer2019deep}, the condition $m\ge \alpha$ is necessary because they approximate any monomial of order $\m$ with $|\m|\le \alpha$ with a two layer neural network, which is impossible when $m<\alpha$. We drop the condition  $m\ge \alpha$ by showing that any monomial of order $\m$ with $|\m|\le \alpha$ can be approximated by deep neural network with a finite number of layers, which depends on $\alpha$.

\section{Application to Statistical Learning Theory}
\label{sec:slt}

In this section, we apply our results about the approximation error of neural networks to the supervised learning problems of regression and classification.
Let $\cX$ be the input space and $\cY$ the output space. Let $\cF$ be a given class of measurable functions from $\cX$ to $\cY$. Let $\P_0$ be the true but unknown data generating distribution on $\cX\times \cY$.
The aim of supervised learning is to find a predictive function that minimizes the population risk $\cR(f):=\E_{( \X,Y)\sim \P_0}\ell(Y, f(\X))$ with respect to a given loss function $\ell$. Since
$\P_0$ is unknown, we cannot directly minimize the population risk, and thus any estimator $\hat{f}$ inevitably has the excess risk which is defined as  $\cR(\hat f)-\inf_{f\in\cF}\cR(f).$ 
For a given sample of size $n$, let $\cF_n$ be a given subset of $\cF$ called a sieve 
and let $(\x_1,y_1),\ldots,(\x_n,y_n)$ be observed (training) data of input--output pairs assumed
to be independent realizations of $(\X,Y)$ following $\P_0.$
Let  $\hat{f}_n$ be an estimated function among $\cF_n$
based on the training data $(\x_1, y_1),\ldots,(\x_n,y_n).$   
The excess risk of $\hat{f}_n$ is decomposed to approximation and estimation errors as
    \begin{equations}
       \cR( \hat{f}_n)-\inf_{f\in\cF}\cR(f) 
       =\underbrace{\sbr{\cR( \hat{f}_n)-\inf_{f\in\cF_n}\cR(f)}}_{\text{Estimation error}}
            + \underbrace{\sbr{\inf_{f\in\cF_n}\cR(f)-\inf_{f\in\cF}\cR(f)}}_{\text{Approximation error}}.
    \end{equations}
    
There is a trade-off between approximation and estimation errors. If the function class $\cF_n$ is sufficiently large to approximate the optimal estimator $f^*:=\argmin_{f\in\cF}\cR(f)$ well, then the estimation error becomes large due to high variance. In contrast, if  $\cF_n$ is small, it leads to low estimation error but it suffers from large approximation error.

One of the advantages of deep neural networks is that we can construct a sieve which has good approximation ability as well as low complexity.
\citet{schmidt2017nonparametric} and Kim et al.~\cite{kim2018fast} proved that a neural network estimator can achieve the optimal balance between the approximation and estimation errors to obtain the minimax optimal convergence rates in regression and classification problems, respectively. But they only considered the ReLU activation function. Based on the results of Theorem~\ref{thm:main1}, we can easily extend their results to general activation functions.

The main tool to derive the minimax optimal convergence rate is that the complexity of a class of functions generated by a deep neural network is not affected much by a choice of an activation function, provided that the activation function is Lipschitz continuous. The function $\sigma:\R\to\R$ is Lipschitz continuous if there is a constant $C_\sigma>0$ such that
    \begin{equation}
        \label{eq:lip}
        |\sigma(x_1)-\sigma(x_2)|\le C_\sigma |x_1-x_2|,
    \end{equation}
for any $x_1,x_2\in\R$. Here, $C_\sigma$ is called the Lipschitz constant. We use the covering number with respect to the $L_\infty$ norm $\norm{\cdot}_\infty$ as a measure of complexity  of function classes. We recall the definition of the covering number. Let $\cF$ be a given class of real-valued functions defined on $\cX$. Let $\delta>0$. A collection $\{f_j\in\cF:j=1,\dots, J\}$ is called a {$\delta$-covering set} of $\cF$ with respect to the $L_\infty$ norm if for all $f\in\cF$, there exists $f_j$ in the collection such that $\|f-f_j\|_\infty\le \delta$. The cardinality of the minimal $\delta$-covering set is called the {$\delta$-covering number} of $\cF$ with respect to the  $L_\infty$ norm  which is denoted by $\cN(\delta, \cF, \|\cdot\|_\infty).$ That is,
    \begin{equation*}
        \cN(\delta, \cF, \|\cdot\|_\infty):=\inf\left\{J\in\bN:\exists f_1, \dots, f_J \mbox{ such that }
\cF\subset\bigcup_{j=1}^J B_\infty(f_j, \delta)\right\},
    \end{equation*}
where $B_\infty(f_j, \delta):=\{f\in\cF:\|f-f_j\|_\infty\le \delta\}$.  The following proposition provides the covering number of a class of functions generated by neural networks.

\begin{Proposition}
\label{prop:entropy}
Assume that the activation function $\sigma$ is Lipschitz continuous with the Lipschitz constant $C_\sigma$. Consider a class of functions generated by a deep neural network
    \begin{equation*}
        \cF_{d,1}(L,N,S,B):=\cbr{N_\sigma(\cdot|\btheta):\btheta\in\Theta_{d,1}(L,N,S,B)}.
    \end{equation*}
    
For any $\delta>0$,
    \begin{equations}
    \log \cN&\left(\delta, \cF_{d,1}(L,N,S,B), \|\cdot\|_\infty\right)
    \le 2L(S+1)\log\left(\delta^{-1}C_\sigma L(N+1)(B\vee1)\right),
    \end{equations}
where $B\vee1:=\max\{B, 1\}$.
\end{Proposition}

{The result in Proposition \ref{prop:entropy} is very similar to the existing results in literature, e.g., Theorem 14.5 of~\citep{anthony2001neural}, Lemma 5 of \citep{schmidt2017nonparametric} and Lemma 3 of \citep{suzuki2018adaptivity}. We employ similar techniques used in \citep{anthony2001neural, schmidt2017nonparametric, suzuki2018adaptivity} to obtain the version presented here. We give the proof of this proposition in Appendix \ref{sec:proofcovering}.}

All of the activation functions considered in Section \ref{sec:act} except RePU satisfy the  Lipschitz condition~(\ref{eq:lip}) and hence Proposition \ref{prop:entropy} can be applied. An interesting implication of Proposition \ref{prop:entropy} is that
the complexity of the function class generated by deep neural networks is not affected by the choice of an activation function. Hence, the remaining step to derive the convergence rate of a neural network estimator is that approximation accuracies by various activation functions are the same as that of the ReLU neural network.

\subsection{Application to Regression}

First we consider the regression problem. For simplicity, we let $\cX=[0,1]^d.$   
Suppose that the generated model is $Y|\X=\x\sim\N(f_0(\x),1)$ for some $f_0:[0,1]^d\to\R.$ 
The performance of an estimator is measured by the $L_2$ risk $\cR_{2, f_0}(f)$, which is defined by 
	\begin{equation*}
	\cR_{2, f_0}(f):=\E_{f_0, \P_\x}(Y-f(\X))^2:=\E_{Y|\X\sim\N(f_0(\X),1),\X\sim\P_\x}(Y-f(\X))^2,
	\end{equation*}
 where $\P_\x$ is the marginal distribution of $\X.$ The following theorem proves that the optimal convergence rate is obtained by the deep neural network estimator of the regression function $f_0$ for a general activation function.

\begin{Theorem}
\label{thm:reg}
Suppose that the activation function $\sigma$ is either piecewise linear or locally quadratic satisfying the  Lipschitz condition (\ref{eq:lip}). Then there are universal positive constants $L_0$, $N_0$, $S_0$ and $B_0$ such that the deep neural network estimator obtained by
    \begin{equation*}
        \hat{f}_n\in\argmin_{f\in \cF_{\sigma,n}}\sum_{i=1}^n\del{y_i-f(\x_i)}^2,
    \end{equation*}
with
    \begin{align*}
        \cF_{\sigma,n}:=\Big\{N_\sigma(\cdot|\btheta):\norm{N_\sigma(\cdot|\btheta)}_\infty\le 2R, 
        \btheta\in\Theta_{d,1}\del{L_0\log n, N_0n^{\frac{d}{2\alpha+d}}, S_0n^{\frac{d}{2\alpha+d}}\log n, B_0n^\kappa} \Big\}
    \end{align*}
for some $\kappa>0$ satisfies  
    \begin{equation*}
    \sup_{f_0\in \cH^{\alpha, R}([0,1]^d)}
    \E\left[\cR_{2, f_0}(\hat{f}_n)- \inf_{f\in\cF}\cR_{2, f_0}(f)\right]\le C
n^{-\frac{2\alpha}{2\alpha +d}}\log^3n,
    \end{equation*}
for some universal constant $C>0$, where the expectation is taken over the training data.
\end{Theorem}

\subsection{Application to Binary Classification}

The aim of the binary classification is to find a classifier that predicts the label $y\in\{-1,1\}$ for any input $\x\in [0,1]^d$.  An usual assumption on the data generating process is that $Y|\X=\x \sim 2\mathsf{Bern}(\eta(\x))-1$ for some $\eta:[0,1]^d\to[0,1]$, where $\mathsf{Bern}(p)$ denotes the Bernoulli distribution with parameter $p$.
Note that  $\eta(\x)$ is  the conditional probability function $\P_0(Y=1|\X=\x).$ A common approach is, instead of finding a classifier directly, to construct a real valued function $f$, a so-called classification function, and predict the label $y$ based on the sign of $f(\x)$. The performance of a classification function is measured by the misclassification error $\cR_{01, \eta}(f)$, which is defined by
	\begin{equation*}
	\cR_{01, \eta}(f):=\E_{\eta, \P_\x}\mathds{1}(Y f(\X)<0)
	:=\E_{Y|\X \sim 2\mathsf{Bern}(\eta(\X))-1, \X\sim \P_\x}\mathds{1}(Y f(\X)<0).
	\end{equation*}
	
It is well known that the convergence rate of the excess risk for classification is faster than that of regression when the conditional probability function $\eta(\x)$ satisfies the following condition: there is a constant $q\in[0,\infty]$ such that for any sufficiently small $u>0$, we~have
    \begin{equation}
    \label{eq:noise}
        \P_\x\del{|\eta(\X)-1/2|<u}\le u^q.
    \end{equation}
    
This condition is called the Tsybakov noise condition and $q$ is called the noise exponent \cite{mammen1999smooth,tsybakov2004optimal}. 
When $q$ is larger, the classification task is easier since the probability of generating vague samples become smaller.  The following theorem proves that the optimal convergence rate can be obtained by the deep neural network estimator with an activation function considered in Section \ref{sec:act}.
As is done by \cite{kim2018fast}, we consider the hinge loss $\ell_{\mathsf{hinge}}(z):=\max\{1-z,0\}.$

\begin{Theorem}
\label{thm:cla}
Assume the Tsybakov noise condition (\ref{eq:noise}) with the noise exponent $q\in[0,\infty]$. Suppose that the activation function $\sigma$, which is either piecewise linear or locally quadratic satisfying the Lipschitz condition (\ref{eq:lip}), is used for all hidden layers except the last one and the ReLU activation function is used for the last hidden layer. Then there are universal positive constants $L_0$, $N_0$, $S_0$ and $B_0$ such that the deep neural network estimator obtained by
    \begin{equation*}
        \hat{f}_n\in\argmin_{f\in \cF_{\sigma,n}}\sum_{i=1}^n\ell_{\mathsf{hinge}}(y_if(\x_i)),
    \end{equation*}
with
    \begin{align*}
        \cF_{\sigma,n}:=\Big\{N_\sigma(\cdot|\btheta):\norm{N_\sigma(\cdot|\btheta)}_\infty\le 1, 
        \btheta\in\Theta_{d,1}\del{L_0\log n, N_0n^{\nu}\log^{-3\nu}n, S_0n^{\nu}\log^{-3\nu+1}n, B_0n^\kappa} \Big\},
    \end{align*}
for $\nu:=d/\cbr{\alpha(q+2)+d}$ and some $\kappa>0$ satisfies  
    \begin{equation*}
    \sup_{\eta\in \cH^{\alpha, R}([0,1]^d)}
    \E\left[\cR_{01, \eta}(\hat{f}_n)- \inf_{f\in\cF}\cR_{01, \eta}(f)\right]\le C
\left(\frac{\log^3 n}{n}\right)^{\frac{\alpha (q+1)}{\alpha (q+2)+d}},
    \end{equation*}
for some universal constant $C>0$, where the expectation is taken over the training data.
\end{Theorem}

Note that the Bayes classifier $f^*:=\argmin_{f\in\cF}\cR_{01,\eta}(f)$ is given by
    \begin{equation*}
        f^*(\x)=2\indi{2\eta(\x)-1\ge0}-1,
    \end{equation*}
which is an indicator function. Since a neural network with the ReLU activation function can approximate indicator functions well \citep{petersen2018optimal, imaizumi2018deep, kim2018fast}, we use the ReLU activation function in the last layer in order to approximate the Bayes classifier more precisely and thus
to achieve the optimal convergence rate.

\section{Conclusions}
\label{sec:con}

In this study, we established the upper bounds of the required depth, width and sparsity of deep neural networks to approximate any H\"older continuous function for the general classes of activation functions. These classes of activation functions include most of the popularly used activation functions. The derived upper bounds of the depth, width and sparsity are optimal in a sense that they are equivalent to the lower bounds up to logarithmic factors. We used this generalization of the approximation error analysis to extend the statistical optimality of the deep neural network estimator in regression and classification problems, where the activation function is other than the ReLU.

Our construction of neural networks for approximation reveals that the piecewise linear activation functions are more efficient in approximating local basis functions while locally quadratic activation functions are more efficient in approximating polynomials. Hence if the activation function has both piecewise linear region and locally quadratic region, we could have a better approximation result. We leave the development of such activation functions as a future work.

\vspace{6pt}

\appendix
\section{Proof of Theorem \ref{thm:main1}}
\vspace{-6pt}
\subsection{Proof of Theorem \ref{thm:main1} for Piecewise Linear Activation Functions}
\label{sec:proof1}

The main idea of the proof is that any deep neural network with the ReLU activation function can be exactly reconstructed by a neural network with a piecewise activation function whose proof is in the next lemma that is a slight modification of Proposition 1 (b) of \citep{yarotsky2017error}.

\begin{Lemma}
\label{lem:piece}
Let $\sigma$ be an any continuous peicewise linear activation function, and $\rho$ be the ReLU activation function. Let $\btheta\in \Theta_{d,1}(L, N, S, B)$. Then there exists $\btheta^*\in \Theta_{d,1}(L, 2N, 4S+2LN+1, C_1 B)$ such that
    $$\sup_{\x\in[0,1]^d}\abs{N_\sigma(\x|\btheta^*)-N_{\rho}(\x|\btheta)}=0,$$
where $C_1>0$ is a constant depending on the activation function $\sigma$.
\end{Lemma}

\begin{proof}
Let $a$ be any break point of $\sigma$. Note that $\sigma(a-)\neq \sigma(a+)$. Let $r_0$ be the distance between $a$ and the closest other break point. Then $\sigma$ is linear on $[a-r_0,a]$ and $[a,a+r_0]$. Then for any $r>0$, the ReLU activation function $\rho(x):=(x)_+$ is expressed as
    \begin{equations}
    \label{eq:reluex}
        \rho(x)&=\frac{\sigma\del{a+\frac{r_{0}}{2 r} x}-\sigma\del{a-\frac{r_{0}}{2}+\frac{r_{0}}{2r} x}-\sigma(a)+\sigma\del{a-\frac{r_{0}}{2}}}{\del{\sigma'(a+)-\sigma'(a-)} \frac{r_{0}}{2r}}\\
        &=: u_1\sigma\del{a+\frac{r_{0}}{2 r} x}+u_2\sigma\del{a-\frac{r_{0}}{2}+\frac{r_{0}}{2r} x} + v
    \end{equations}
for any $x\in[-r,r]$, where we define $u_1:=1/\del[0]{\del[0]{\sigma'(a+)-\sigma'(a-)} \frac{r_{0}}{2r}}$, $u_2:=-1/\del[0]{\del[0]{\sigma'(a+)-\sigma'(a-)} \frac{r_{0}}{2r}}$ and $v:=\del[0]{-\sigma(a)+\sigma\del{a-r_0/2}}/\del[0]{\del[0]{\sigma'(a+)-\sigma'(a-)}\frac{r_{0}}{2r}}$. 

Let $\btheta\equiv((\W_1,\b_1),\dots,(\W_{L+1}, \b_{L+1}))\in \Theta_{d,1}(L, N, S, B)$ be given. 
Since both input $\x\in[0,1]^d$ and the network parameter $\btheta$ are bounded, we can take a sufficiently large $r$ so that Equation~(\ref{eq:reluex}) holds for any hidden nodes of the network $\btheta$. 
We consider the deep neural network $\btheta^*\equiv((\W_1^*,\b_1^*)$, $\dots,(\W_{L+1}^*, \b_{L+1}^*))\in \Theta_{d,1}(L, 2N),$ where we set
  \begin{align*}
        \W_{l}^*&:=\frac{r_{0}}{2r}\begin{pmatrix}u_1\W_{l} & u_2\W_{l} \\ u_1\W_{l} & u_2\W_{l} \end{pmatrix}\in\R^{2n_{l}\times 2n_{l-1}}, \\
        \b_{l}^*&:=\begin{pmatrix} a\one_{n_l}+\frac{r_0}{2r}(v\W_l\one_{n_{l-1}}+\b_l)\\   \del{a-\frac{r_0}{2}}\one_{n_l}++\frac{r_0}{2r}(v\W_l\one_{n_{l-1}}+\b_l)\end{pmatrix} \in\R^{2n_{l}},
    \end{align*}
for $l=1,\dots, L$ and 
    \begin{equation*}
        \W_{L+1}^*:=\begin{pmatrix}u_1\W_{L+1} & u_2\W_{L+1} \end{pmatrix},\quad
        \b_{L+1}^*:=v.
    \end{equation*}
Here, $\one_n$ denotes the $n$-dimensional vector of $1's$. Then by Equation (\ref{eq:reluex}) and some algebra, we have that $N_\sigma(\x|\btheta^*)=N_\rho(\x|\btheta)$ for any $\x\in[0,1]^d$. For the sparsity of $\btheta^*$, we note that
        \begin{equation*}
            \abs{\text{vec}(\W_l^*)}_0+\abs{\b_l^*}_0\le 4\abs{\text{vec}(\W_l)}_0 + 2n_l
        \end{equation*}
    which implies that $\abs{\btheta^*}_0\le 4\abs{\btheta}_0+2L(\btheta)n_{\max}(\btheta)+1$.
\end{proof}

Thanks to Lemma \ref{lem:piece}, to prove Theorem \ref{thm:main1} for piecewise linear activation functions, it suffices to show the approximation ability of the ReLU networks, which is already done by \citep{schmidt2017nonparametric} as in the next lemma.

\begin{Lemma}[Theorem 5 of \citep{schmidt2017nonparametric}]
\label{lem:reluapprox}
Let $\rho$ be the ReLU activation function. For any $f\in\cH^{\alpha, R}([0,1]^d)$ and any integers $m\ge 1$ and $M\ge\max\cbr{(\alpha+1)^d, (R+1)\e^d}$, there exists a network parameter $\btheta\in\Theta_{d,1}(L,N, S, 1)$ such that 
    \begin{equation}
        \sup_{\x\in[0,1]^d}\abs{N_\rho(\x|\btheta)-f(\x)}\le (2R+1)(1+d^2+\alpha^2)6^dM2^{-m} +R3^\alpha M^{-\alpha/d},
    \end{equation}
where $L=8+(m+5)(1+ \lceil \log_2 (d\vee \alpha)\rceil)$, $N= 6(d+\ceil{\alpha})M$, and $S=141(d+\alpha+1)^{3+d}M(m+6)$.
\end{Lemma}

Theorem \ref{thm:main1} for piecewise linear activation functions is a direct consequence of Lemmas \ref{lem:piece} and~\ref{lem:reluapprox}, which is summarized as follows. 
 
\begin{proof}[Proof of Theorem \ref{thm:main1} for piecewise linear activation functions]
Let $\rho$ be the ReLU activation function. By letting $M=3^{d}(2R)^{d/\alpha}\epsilon^{-d/\alpha}$ and $m=\log_2\left(2(2R+1)(1+d^2+\alpha^2)18^{d}(2R)^{d/\alpha}\epsilon^{-d/\alpha-1}\right),$ Lemma~\ref{lem:reluapprox} implies that there exists a network parameter $\btheta'$ such that $\sup_{\x\in[0,1]^d}|N_\rho(\x|\btheta')-f(\x)|\le\epsilon$ with $L(\btheta')\le L_0'\log\left(1/\epsilon\right)$, $n_{\max}(\btheta')\le N_0'\epsilon^{-d/\alpha}$ and $|\btheta'|_0\le S_0'\epsilon^{-d/\alpha}\log\left(1/\epsilon\right)$ for some positive constants $L_0'$, $N_0'$, and $S_0'$ depending only on $\alpha$, $d$ and $R$. Hence by Lemma \ref{lem:piece}, there is a network parameter $\btheta$ producing the same output of the ReLU neural network $N_\rho(\cdot|\btheta)$ with $L(\btheta)=L(\btheta')$,  $n_{\max}(\btheta)=2n_{\max}(\btheta')$,  $|\btheta|_0\le 4|\btheta'|_0+2L(\btheta')n_{\max}(\btheta')+1\le S_0\epsilon^{-d/\alpha}\log\left(1/\epsilon\right)$ and $|\btheta|_\infty\le B_0|\btheta'|_\infty$ for some $S_0>0$ depending only on  $\alpha$, $d$, $R$ and $\sigma$, and some $B_0>0$ depending only on $\sigma$, which completes the proof.
\end{proof}

\subsection{Proof of Theorem \ref{thm:main1} for Locally Quadratic Activation Functions}

\begin{Lemma}
\label{lem:basic}
Assume that an activation function $\sigma$ is locally quadratic. There is a constant $K_0$ depending only on the activation function such that for any $K>K_0$ the following results hold.
    \begin{enumerate}[label=(\alph*)]
        \item \label{lem:sq} There is a neural network $\btheta_{2}\in\Theta_{1,1}(1, 3)$ with $\abs{\btheta_2}_\infty \le K^2$ such that 
            \begin{equation*}
            \sup_{x\in[-1,1]}\abs{N_\sigma(x|\btheta_2)-x^2}\le \frac{C_1}{K},
            \end{equation*}
        where $C_1>0$ is a constant depending only on $\sigma$.
        
        \item \label{lem:times} Let $A>0$. There is a neural network  parameter $\btheta_{\times,A}\in\Theta_{2,1}(1, 9)$ with $\abs{\btheta_{\times,A}}_\infty \le \max\{K^2,2A^2\}$ such that
            \begin{equation*}
            \sup_{\x\in[-A,A]^2}\abs{N_\sigma(\x|\btheta_{\times,A})-x_1x_2}\le \frac{6A^2C_1}{K}.
            \end{equation*}
        
        \item \label{lem:mono} Let $\alpha$ be a positive integer. For any multi-index $\m\in\bN_0^d$ with $|\m|\le \alpha$, there is a network parameter  $\btheta_{\m}\in\Theta_{d,1}(\ceil{\log_2\alpha}, 9\alpha)$ with $\abs{\btheta_\m}_\infty \le \max\{K^2, C_2\}$ such that
            \begin{equation*}
                \sup_{\x\in[0,1]^d}\abs{N_\sigma(\x|\btheta_{\m})-\x^\m}\le \frac{C_3}{K},
            \end{equation*}
        for some positive constants $C_2$ and $C_3$ depending only on $\sigma$ and $\alpha$.
        
        \item  \label{lem:sqrt} There is a network parameter $\btheta_{1/2}\in\Theta_{1,1}(\ceil{\log K}, 15)$  with $\abs{\btheta_{1/2}}_\infty \le \max\{K^2,C_4\}$ such that 
            \begin{equation*}
                \sup_{x\in[0,2]}\abs{N_\sigma(x|\btheta_{1/2})-\sqrt{x}}\le
                C_5\frac{\log K}{K}
          \end{equation*}
        for some positive constants $C_4$ and $C_5$  depending only on $\sigma$.
        
        \item \label{lem:abs} There is a network parameter $\btheta_{\textup{abs}}\in\Theta_{1,1}(\ceil{\log K}, 15)$  with $\abs{\btheta_{\textup{abs}}}_\infty \le \max\{K^2,C_6\}$ such that 
            \begin{equation*}
                \sup_{x\in[-1,1]}\abs{N_\sigma(x|\btheta_{\textup{abs}})-|x|}\le \frac{C_7}{\sqrt{K}},
            \end{equation*}
         for some positive constants $C_6$ and $C_7$  depending only on $\sigma$.   
    \end{enumerate}
\end{Lemma}

\begin{proof}
Recall that there is an interval $(a,b)$ on which $\sigma(x)$ is three times continuously differentiable with bounded derivatives and there is $t\in(a,b)$ such that $\sigma'(t)\neq0 $ and $\sigma''(t)\neq0 $
\medskip

    {Proof of} \ref{lem:sq}. Take $K$ large so that $2/K<\min\{|t-b|,|t-a|\}$.  Consider a neural network
        \begin{equation}
        \label{eq:sqnn}
             N_\sigma(x|\btheta_2):=\sum_{k=0}^2(-1)^{k-1}\frac{K^2}{\sigma''(t)}\binom{2}{k}\sigma\del{\frac{k}{K} x+t}.
        \end{equation}
        
      Since  $\sigma$ is three times continuously differentiable on $(a,b)$ and $(k-1)x/K+t\in(a,b)$ if $x\in[0,1]$, it can be expanded in the Taylor series with Lagrange remainder around $t$ to have 
        \begin{align*}
            N_\sigma(x|\btheta_2)&=\frac{K^2}{\sigma''(t)}\sum_{k=0}^2(-1)^{k}\binom{2}{k}
                    \cbr{\sigma(t)+\sigma'(t)\frac{kx}{K}+\frac{\sigma''(t)}{2}\frac{(kx)^2}{K^2} +\frac{\sigma''(\xi_k)}{6}\frac{(kx)^3}{K^3} }\\
            &=\frac{K^2}{\sigma''(t)}\cbr{\sigma''(t)\frac{x^2}{K^2}             +\sum_{k=1}^2(-1)^{k}\binom{2}{k}\frac{\sigma'''(\xi_k)}{6}\frac{(kx)^3}{K^3} }   \\  &=x^2+\frac{x^3}{6K\sigma''(t)}\sum_{k=1}^2(-1)^{k}k^3\binom{2}{k}\sigma'''(\xi_k),
        \end{align*}
    where $\xi_k\in[t-k|x|/K, t+k|x|/K]\subset(a,b).$ Since the third order derivative is bounded on $(a,b)$, we get the desired assertion by retaking $K\leftarrow\sqrt{2/\sigma''(t)}K$.
  \medskip
    
    {Proof of} \ref{lem:times}. The proof can be done straightforwardly by the polarization type identity:
        \begin{equation*}
            x_1x_2=2A^2\cbr{\del{\frac{x_1+x_2}{2A}}^2-\del{\frac{x_1}{2A}}^2-\del{\frac{x_1}{2A}}^2}.
        \end{equation*}
        
        We construct the network as
        \begin{equation}
        \label{eq:thetatimes}
            N_\sigma(\x|\btheta_{\times,A})
            :=2A^2\cbr{N_\sigma\del{\frac{x_1+x_2}{2A}\big|\btheta_2} -N_\sigma\del{\frac{x_1}{2A}\big|\btheta_2}-N_\sigma\del{\frac{x_2}{2A}\big|\btheta_2}},
        \end{equation}
    where $\btheta_2$ is defined in (\ref{eq:sqnn}). Since $(x_1+x_2)/2A, x_1/2A, x_2/2A \in [-1,1]$ for $\x\in[-A,A]^2$, the triangle inequality implies that $|N_\sigma(\x|\btheta_{\times,A})-x_1x_2|\le6A^2C_1/K$.
\medskip
    
    {Proof of} \ref{lem:mono}. Let $q := \ceil{\log_2\alpha}$. We construct $\btheta_\m$ as follows. Fix $\x\equiv(x_1,\dots, x_d)\in[0,1]^d$. We first consider the affine map that transforms $(x_1,\dots, x_d)$ to $\z\in[0,1]^{2^q}$ which is given by
        \begin{equation*}
            \z:=\del[0]{\underbrace{x_1,\dots,x_1}_{m_1\mbox{ times}},
            \underbrace{x_2,\dots,x_2}_{m_2\mbox{ times}},
            \dots, 
            \underbrace{x_d,\dots,x_d}_{m_d\mbox{ times}},
             \underbrace{1,\dots,1}_{2^q-|\m|\mbox{ times}}}.
        \end{equation*}
        
     The first hidden layer of $\btheta_\m$ pairs neighboring entries in $\z$ and applies the network $\btheta_{\times,A_1}$ defined in \ref{lem:times} with $A_1=1$ to each pair. That is,  the first hidden layer of $\btheta_\m$ produces
        \begin{equation*}
            \cbr{g_{1,j}:=N_\sigma((z_{2j-1}, z_{2j})|\btheta_{\times, 1}):j=1,\dots, 2^{q-1}}.
        \end{equation*}
     Note that $\sup_{1\le j\le 2^{q-1}}|g_{1,j}-z_{2j-1}z_{2j}|\le 6C_1/K$ and $\sup_{1\le j\le 2^{q-1}}|g_{1,j}|\le 6C_1/K+1$, where $6C_1/K+1$ can be bounded by some constant $A_2>1$  depending only on $C_1$ and $K_0$. Then the second hidden layer of $\btheta_\m$ pairs neighboring entries of $\cbr{g_{1,j}:j=1,\dots, 2^{q-1}}$ and applies $\btheta_{\times,A_2}$ to each pair to have
       \begin{equation*}
            \cbr{g_{2,j}:=N_\sigma((g_{1,2j-1}, g_{1,2j})|\btheta_{\times, A_2}):j=1,\dots, 2^{q-2}}.
        \end{equation*}
    Note that $\sup_{1\le j\le 2^{q-2}}|g_{2,j}-g_{1,2j-1}g_{1,2j}|\le 6C_1A_2^2/K$ and $\sup_{1\le j\le 2^{q-2}}|g_{2,j}|\le 6C_1A_2^2/K+1\le A_3$ for some $A_3>1$ depending only on $C_1$ and $K_0$. We repeat this procedure to produce $\cbr{g_{k,j}:j=1,\dots, 2^{q-k}}$ for $k=3,\dots,q$ with 
        \begin{equation*}
            \sup_{1\le j\le 2^{q-k}}\abs{g_{k,j}-g_{k-1,2j-1}g_{k-1,2j}}\le \frac{6C_1A_k^2}{K},\quad
            \sup_{1\le j\le 2^{q-k}}\abs{g_{k,j}}\le  A_{k+1},
        \end{equation*}
    for some $A_{k+1}>1$, and we set $N_\sigma(\x|\btheta_\m)$ equal to $g_{q,1}$.
    
    By applying the triangle inequality repeatedly, we have
          \begin{align*}
            \abs{g_{q,1}-\x^{\m}}&\le  \abs{g_{q,1}-g_{q-1,1}g_{q-1,2}} + \abs{g_{q-1,1}-\prod_{j=1}^{2^{q-1}}z_j}\abs{g_{q-1,2}} + \abs{g_{q-1,2}-\prod_{j=2^{q-1}+1}^{2^q}z_j}\abs{\prod_{j=1}^{2^{q-1}}z_j}\\
            &\le\frac{6C_1A_{q}^2}{K} + A_{q}\abs{g_{q-1,1}-\prod_{j=1}^{2^{q-1}}z_j}+\abs{g_{q-1,2}-\prod_{j=2^{q-1}+1}^{2^q}z_j}\\
            &\le\frac{6C_1A_{q}^2}{K} +(A_q+1)\frac{6C_1A_{q-1}^2}{K}
                +A_qA_{q-1}\abs{g_{q-2,1}-\prod_{j=1}^{2^{q-2}}z_j}+A_q\abs{g_{q-2,2}-\prod_{j=2^{q-2}+1}^{2\times2^{q-2}}z_j} \\
            &\quad\quad + A_{q-1}\abs{g_{q-2,3}-\prod_{j=2\times 2^{q-2}+1}^{3\times2^{q-2}}z_j}
                + \abs{g_{q-2,4}-\prod_{j=3\times 2^{q-2}+1}^{4\times2^{q-2}}z_j} \\
            &\le \cdots
            \le \sum_{k=0}^{q-1}\cbr{A_{q-k}^2\prod_{h=q-k+1}^{q}(A_{h}+1)}\frac{6C_1}{K}  \le C_1'\frac{1}{K},
        \end{align*}
    for some $C_1'>0$ depending only on $C_1$, $K_0$ and $q$. Since we set $\x$ arbitrary, the proof is done.
  \medskip
    
     {Proof of} \ref{lem:sqrt}. By \ref{lem:times}, it is easy to verify that there is a network $\btheta_1\in\Theta_{1,1}(1, 6)$ with $|\btheta_1|_\infty\le \max\{K^2, 2\}$ such that     
     $|\sigma(x)-x|\le C_1'/K$ for any $x\in[-1,1]$ and some constant $C_1'>0$.
     The Taylor series with Lagrange remainder around 1 of $\sqrt{x}$ is given by
        \begin{equation*}
            \sqrt{x}=\sum_{k=0}^J \frac{(x-1)^k}{k!} + \frac{1}{(J+1)!}\frac{\d^{J+1}\sqrt{x}}{\d x^{J+1}}\big|_{x=\xi}(x-1)^{J+1},
        \end{equation*}
    where $\xi\in[0,2],$ and thus
        \begin{equation*}
            \sup_{x\in[0,2]}\abs{\sqrt{x}-\sum_{k=0}^J \frac{(x-1)^k}{k!}}\le C_1' \frac{1}{(J+1)!}\le \e\del{\frac{\e}{J+1}}^{J+1}.
        \end{equation*}
    for some $C_1'>0,$ where the last inequality is because $n!\ge(n/\e)^n\e$. 
    
    Now, we will construct a neural network $\btheta_{p,J}$ that approximates the polynomial $\sum_{k=0}^J \frac{(x-1)^k}{k!}$ as follows. The first hidden layer computes $(N_\sigma(x-1|\btheta_2)/2, N_\sigma(x-1|\btheta_1))$ from the input $x$. Then
        \begin{equation*}
            \abs{\del{N_\sigma(x-1|\btheta_2)/2, N_\sigma(x-1|\btheta_1)}-\del{(x-1)^2/2,(x-1)}}_\infty \le C_2'\frac{1}{K},
        \end{equation*}
     for any $x\in[0,1]$ and some constant $C_2'>0$. The next hidden layer computes $(N_\sigma((u,v)|\btheta_{\times, 1+C_2'/K})/3, N_\sigma(u+v|\btheta_1))$ from the input $(u,v)$ from the first hidden layer. Using the triangle inequality, we have that the second hidden layer approximates the vector $((x-1)^3/3!, (x-1)^2/2+(x-1))$ by error $\le  2C_3' /K$ for some $C_3'>0$. Repeating this procedure, we construct the network $\btheta_{p,J}\in\Theta_{1,1}(J, 15)$ which approximates  $\sum_{k=0}^J \frac{(x-1)^k}{k!}$ by error $\le C_4'J/K$ for some $C_4'>0$.  Taking $J=\ceil{\log K}$, we observe that $(\e/J+1)^{J+1}\le (\e/\log K)^{\log K +1}\le \e K/(\log K)^{\log K}\le 1/K$ for all sufficiently large $K$, which implies the desired result.
\medskip
     
    {Proof of} \ref{lem:abs}. Let $\zeta\in(0,1)$. Since for any $x\in\R$,
        \begin{equation*}
            \sqrt{x^2+\zeta^2}-|x|\le \frac{\zeta^2}{ \sqrt{x^2+\zeta^2}+|x|}\le \frac{\zeta^2}{\zeta}= \zeta,
        \end{equation*}
    the function  $\sqrt{x^2+\zeta^2}$ approximates the absolute value function $|x|$ by error $\zeta$. For $\btheta_2$ in \ref{lem:sq} and $\btheta_{1/2}$ in \ref{lem:sqrt}, we have that
       \begin{align*}
           \abs{N_\sigma\del{N_\sigma(x|\btheta_2)+\xi^2\big|\btheta_{1/2}}-|x|}
           \le&  \abs{N_\sigma\del{N_\sigma(x|\btheta_2)+\zeta^2\big|\btheta_{1/2}}-\sqrt{x^2+\zeta^2}}+\zeta \\
           \le &  \abs{N_\sigma\del{N_\sigma(x|\btheta_2)+\zeta^2\big|\btheta_{1/2}}-\sqrt{N_\sigma(x|\btheta_2)+\xi^2}} \\
           &+ \abs{\sqrt{N_\sigma(x|\btheta_2)+\zeta^2}-\sqrt{x^2+\zeta^2}} + \zeta \\
           \le & C_1'\del{\frac{\log K}{K} +\frac{1}{K\zeta}} + \zeta
       \end{align*}
    for some constant $C_1'>0.$
    We now set $\zeta=1/\sqrt{K}$ and $N_\sigma(x|\btheta_{\textup{abs}}):=N_\sigma(N_\sigma(x|\btheta_2)+K^{-1}|\btheta_{1/2})$. Since $(\log K)/K=o(1/\sqrt{K})$, the proof is done.
\end{proof}

\begin{proof}[Proof of Theorem \ref{thm:main1} for locally quadratic activation functions]
Recall that 
    \begin{equation*}
        P_M(\x)=\sum_{\z\in\bG_{d,M}}\sum_{\m\in\bN_0^d:|\m|\le \alpha}\beta_{\z, \m}\x^\m\phi_{\z, M}(\x).
    \end{equation*}
Then by Lemma B.1 of \citep{schmidt2017nonparametric},
    \begin{equation*}
        \sup_{\x\in[0,1]^d}\abs{P_M(\x)-f(\x)}\le RM^{-\alpha}.
    \end{equation*}
    
From the equivalent representation of the ReLU function $(x)_+=(x+|x|)/2$, we can easily check that the neural network $N_\sigma(x|\btheta_{\textup{relu}}):=\del{N(x|\btheta_{\textup{abs}})+N_\sigma(x|\btheta_{1})}/2$ with $\btheta_{\textup{relu}}\in\Theta_{1,1}(\ceil{\log K}, 21)$ approximates the ReLU function by error $\le C_1' /\sqrt{K}$ for some $C_1'>0$, where $\btheta_1\in\Theta_{1,1}(1,6)$ is defined in the proof of \ref{lem:sqrt} of Lemma \ref{lem:basic} and $\btheta_{\textup{abs}}\in\Theta_{1,1}(\ceil{\log K}, 15)$ is defined in \ref{lem:abs} of Lemma \ref{lem:basic}. For $z\in(0,1)$ and $M\in\bN$, we~define
    \begin{equation*}
        N_\sigma(x|\btheta_{\phi, z, M}):=N_\sigma\del{1/M-N_\sigma((x-z)|\btheta_{\textup{abs}})\big|\btheta_{\textup{relu}}}.
    \end{equation*}
    
Then it approximates the function $(1/M-|x-z|)_+$ by error $\le C_2' /\sqrt{K}$ for some $C_2'>0.$  In turn, for $\z\in\bG_{d,M}$, by invoking the similar construction used in \ref{lem:mono} of Lemma \ref{lem:basic} to approximates the product of $d$ components, we can construct the network $\btheta_{\phi,\z, M}\in\Theta_{1,1}(\ceil{\log K}+\ceil{\log_2d}, 21d)$ with $\abs{\btheta_{\phi,\z, M}}_\infty\le C_3' K^2$ for some $C_3'>0$ such that
    \begin{equation*}
       \sup_{\x\in[0,1]^d}\abs{ N(\x|\btheta_{\phi,\z, M})-\prod_{j=1}^d\del{\frac{1}{M}-|x_j-z_j|}_+}\le C_4'\frac{1}{\sqrt{K}},
    \end{equation*}
for some $C_4'>0$. For each $\m\in\bN_0^d$ with $|\m|\le \alpha$, we have the neural network $\btheta_\m$ in  \ref{lem:mono} of Lemma \ref{lem:basic} that approximates $\x^\m$. The number of these networks is $\binom{d+\alpha}{\alpha}$, which is denoted by $A_\alpha$. Also there are $|\bG_{d,M}|=(M+1)^d$ networks $\btheta_{\phi,\z, M}$ for $\z\in\bG_{d,M}$. We need approximation of each product $\x^\m \phi_{\z, M}$, which requires additional $A_\alpha(M+1)^d$ many networks $\btheta_{\times,A}\in\Theta_{2,1}(1,9)$, where $\btheta_{\times,A}$ is defined as in (\ref{eq:thetatimes}) for some $A>1$ not depending on $M$ and $K$.  Finally we construct the output layer which computes the weighted sum of $\cbr{N_\sigma\del{(N_\sigma(\x|\btheta_\m), N_\sigma(\x|\btheta_{\phi, \z, M}))|\btheta_{\times, A}}: \m\in\bN_{0}^d, |\m|\le\alpha, \z\in\bG_{d,M}}$. Letting $\btheta_{f, K, M}$ be the network constructed above, we can check that
    \begin{equation*}
        \sup_{\x\in[0,1]^d}\abs{N(\x|\btheta_{f,K,M})-P_M(\x)}\le C_5'  A_\alpha(M+1)^d\del{ \frac{1}{K}+\frac{1}{\sqrt{K}}}\le C_6' \frac{(M+1)^d}{\sqrt{K}},
    \end{equation*}
for some positive constants $C_5'$ and $C_6'.$ In addition, we have $L(\btheta_{f, K, M})\le 1+ (\ceil{\log K}+\ceil{\log_2(\alpha\vee d)}\le C_7'\ceil{\log K}$ and $n_{\max}(\btheta_{f, K, M})\le C_8' A_\alpha (M+1)^d$ for some positive constants $C_7'$ and $C_8'$. For sparsity of the network, we have 
    \begin{align*}
            \abs{\btheta_{f,K,M}}_0 
            &\le A_\alpha(M+1)^d\abs{\btheta_{\times, A}}_0 +(M+1)^d\abs{\btheta_{\phi, \z, M}}_0 + A_\alpha \abs{\btheta_{\m}}_0\\
            &\le C_9' \ceil{\log K}(M+1)^d,
    \end{align*}
for some $C_9'>0.$
Taking $M+1=\epsilon^{-1/\alpha}$ and $K=\epsilon^{-2d/\alpha-2}$,  we have 
    \begin{equation*}
        \btheta_{f,K,M}\in\Theta\del{L_0\log (1/\epsilon), N_0\epsilon^{-d/\alpha}, S_0\epsilon^{-d/\alpha}\log (1/\epsilon), B_0\epsilon^{-4(d/\alpha+1)}},
    \end{equation*}
so that $\norm{P_M-N_\sigma(\cdot|\btheta_{f,K,M})}_\infty\le C_{10}' \epsilon$ for some $C_{10}'>0.$ 
Since $\norm{f-P_M}_\infty\le R M^{-\alpha} \le C_{11}' \epsilon$ for some 
$C_{11}'>0,$ the proof is done.
\end{proof}

\section{Proofs of Proposition \ref{prop:entropy}}
\label{sec:proofcovering}

\begin{proof}
Given a deep neural network $\btheta=((\W_1,\b_1),\dots, (\W_{L+1}, \b_{L+1}))\in\Theta_{d,1}(L,N,S,B)$, we define $\check{N}_{l,\sigma, \btheta}:\R^d\to\R^{n_{l-1}}$ and $ \hat{N}_{l,\sigma, \btheta}:\R^{n_{l}}\to\R$ as
    \begin{align*}
        \check{N}_{l,\sigma, \btheta}(\x)&:=\sigma_{l-1}\circ\sA_{l-1}\circ\cdots \circ\sigma_1\circ\sA_1(\x),\\
        \hat{N}_{l,\sigma, \btheta}(\x)&:=\sA_{L+1}\circ\sigma_L\circ\sA_{L}\circ\cdots  \sigma_l\circ\sA_l\circ\sigma_{l-1}(\x),
    \end{align*}
for $l\in 2,\dots, L$, where $\sA_l\x=\W_l\x+\b_l$. Corresponding to the last and first layer, we define   $\check{N}_{1,\sigma, \btheta}(\x)=\x$ and $\hat{N}_{L+1,\sigma, \btheta}(\x)=\x$. Note that $N_{\sigma}(\x|\btheta)=\hat{N}_{l+1,\sigma, \btheta}\circ \sA_l \circ \check{N}_{l,\sigma, \btheta}(\x) $. For given $\delta>0,$
let $\btheta=((\W_1,\b_1),\dots, (\W_{L+1}, \b_{L+1}))\in\Theta_{d,1}(L,N,S,B)$ and $\btheta^*=((\W_1^*,\b_1^*),\dots, (\W_{L+1}^*, \b_{L+1}^*))\in\Theta_{d,1}(L,N,S,B)$ be two neural network parameter such that  $\abs{\text{vec}(\W_l-\W_l^*)}_\infty\le \delta$ and  $\abs{\b_l-\b_l^*}_\infty\le  \delta$ for $l=1,\dots, L+1$. Let $C_\sigma$ be the Lipschitz constant of $\sigma$. We observe that
    \begin{align*}
        \norm{ \check{N}_{l,\sigma, \btheta}}_\infty 
        &\le C_\sigma \del{NB\norm{ \check{N}_{l-1,\sigma, \btheta}}_\infty +B}\\
        &\le  C_\sigma (B\vee1) (N+1)\norm{ \check{N}_{l-1,\sigma, \btheta}}_\infty        \\
        &\le  \cbr{C_\sigma (B\vee1) (N+1)}^{l-1},
    \end{align*}
and similarly, $ \norm{\hat{N}_{l,\sigma, \btheta}}_\infty \le (C_\sigma BN)^{L-l+1}$. Letting $\sA_l^*\x=\W_l^*\x+\b_l^*,$ we have
    \begin{align*}
        \norm{N_{\sigma}(\cdot|\btheta)-N_{\sigma}(\cdot|\btheta^*)}_\infty
        &\le \norm{\sum_{l=1}^L\sbr{ \hat{N}_{l+1,\sigma, \btheta^*}\circ \sA_l \circ \check{N}_{l,\sigma, \btheta}(\cdot) - \hat{N}_{l+1,\sigma, \btheta^*}\circ \sA_l^* \circ \check{N}_{l,\sigma, \btheta}(\cdot) }}_\infty\\
        &\le \sum_{l=1}^L (C_\sigma BN)^{L-l}\norm{(\sA_l-\sA_l^*) \circ \check{N}_{l,\sigma, \btheta}(\cdot)}_\infty\\
        &\le \sum_{l=1}^L (C_\sigma BN)^{L-l}  \delta\cbr{C_\sigma (B\vee1) (N+1)}^{l-1}\\
        &\le  \delta L \cbr{C_\sigma (B\vee1) (N+1)}^{L}.
    \end{align*}
    
Thus, for a fixed sparsity pattern  (i.e., the location of nonzero elements in $\btheta$), the covering number is bounded by $\sbr{\delta/L\cbr{C_\sigma(B\vee1) (N+1)}^{L}}^{-S}$. Since the number of the sparsity patterns is bounded by $\binom{(N+1)^L}{S}\le (N+1)^{LS}$, the log of  covering number is bounded above by
    \begin{equation*}
         \log \del{(N+1)^{LS}\sbr{\frac{L\cbr{C_\sigma (B\vee1) (N+1)}^{L}}{\delta}}^{S}}
         \le 2LS \log \del{\frac{C_\sigma L (B\vee1) (N+1)}{\delta}},
    \end{equation*}
which completes the proof.
\end{proof}

\section{Proof of Theorem \ref{thm:reg}}

The proof Theorem  \ref{thm:reg} is based on the following oracle inequality.

\begin{Lemma}[Lemma 4 of \citep{schmidt2017nonparametric}]
\label{lem:oracle}
Assume that $Y|\X=\x\sim\N(f_0(\x),1)$ for some $f_0$ with $\norm{f_0}_\infty \le R$. Let $\cF^\dagger$ be a given function class from $[0,1]^d$ to $[-2R,2R]$, and let $\hat{f}$ be any estimator in $\cF^\dagger$. Then for any $\delta\in(0,1]$, we have
    \begin{align*}
        \E\sbr{\E_{\X\sim\P_\x}\del{\hat{f}(\X)-f_0(\X)}^2} 
		\le 4\Big[&\inf_{f \in \cF^\dagger} \E_{\X\sim \P_\x}\left(f(\X)-f_{0}(\X)\right)^{2} \\
        &+(4R)^{2} \frac{18\log \cN(\delta, \cF^\dagger, \|\cdot\|_\infty)+72}{n }+32 \delta(4 R)+\Delta_{n}\Big],
    \end{align*}
with  
    \begin{equation*}
        \Delta_n:=\E\left[\frac{1}{n} \sum_{i=1}^{n}\left(Y_{i}-\hat{f}\left(\X_{i}\right)\right)^{2}-\inf_{f \in \cF^\dagger} \frac{1}{n} \sum_{i=1}^{n}\left(Y_{i}-f\left(\X_{i}\right)\right)^{2}\right],
    \end{equation*}
where the expectations are taken over the training data.
\end{Lemma}

\begin{proof}[Proof of Theorem \ref{thm:reg}]
We apply Lemma \ref{lem:oracle} to $\cF^\dagger=\cF_{\sigma,n}$ and $\hat{f}=\hat{f}_n\in\argmin_{f\in \cF_{\sigma,n}}\sum_{i=1}^n\del{y_i-f(\x_i)}^2$. By definition of $\hat{f}_n$, we have $\Delta_n=0$. Also it can be easily verified that $f_0=\argmin_{f\in\cF}\cR_{2,f_0}(f)$ and $\E_{f_0, \P_\x}\del{\hat{f}_n(\X)-f_0(\X)}^2=\cR_{2, f_0}(\hat{f}_n)- \cR_{2,f_0}(f_0).$ Set $\delta=1/n$. By Proposition \ref{prop:entropy},
    \begin{equation*}
        \log\cN\del{\frac{1}{n}, \cF_{\sigma,n}, \|\cdot\|_\infty}
        \le C_1' n^{\frac{d}{2\alpha+d}}\log^3 n,
    \end{equation*}
for some $C_1'>0$. If a function $f_n$ is approximates $f_0$ by error $\epsilon$ which is sufficeintly small, then $\norm{f_n}_\infty\le 2R$ since $\norm{f_0}_\infty\le R$. Now, Theorem \ref{thm:main1} implies that there is $f_n\in \cF_{\sigma,n} $ such that
    \begin{align*}
      \E_{f_0, \P_\x}\left(f_n(\X)-f_{0}(\X)\right)^{2}
      &\le C_2' \sup_{\x\in[0,1]^d}\abs{f_n(\x)-f_{0}(\x)}^{2} \\
      &\le C_3'\del{\del{n^{\frac{d}{2\alpha+d}}}^{-d/\alpha}}^2 = C_3'n^{-\frac{2\alpha}{2\alpha+d}},
    \end{align*}
which completes the proof.
\end{proof}

\section{Proof of Theorem \ref{thm:cla}}

For a given real-valued function $f$, let $\cR_{\mathsf{hinge}, \eta}(f):=\E_{Y|\X \sim 2\mathsf{Bern}(\eta(\X))-1, \X\sim \P_\x}\ell_{\mathsf{hinge}}(Yf(\X))$, which we call the hinge risk. 
The proof of Theorem \ref{thm:cla} is based on the following theorem, which is given in \citep{kim2018fast}. 

\begin{Lemma}[Theorem 6 of \citep{kim2018fast}]
\label{lem:hinge}
Assume that $\eta(\x)$ satisfies the Tsybakov noise condition (\ref{eq:noise}) with the noise exponent $q\in[0,\infty]$. Assume that there exists a sequence $(\delta_n)_{n\in\bN}$ such that

    \begin{itemize}
        \item  there exists a sequence of classes of functions $\{\cF_n\}_{n\in\bN}$ with $\sup_{n\in\bN}\sup_{f\in\cF_n}\norm{f}_\infty\le F$ for some $F>0$ such that 
        there is $f_n\in \cF_n$ with
        $\cR_{\mathsf{hinge}, \eta}(f_n)-\min_{f\in\cF}\cR_{\mathsf{hinge}, \eta}(f)\le C_1\delta_n$
    for some universal constant $C_1>0$;
    
        \item $\log \cN(\delta_n, \cF_n, \norm{\cdot}_\infty)\le C_2 n\delta_n^{(q+2)/(q+1)}$  for some universal constant $C_2>0$.
    \end{itemize}
Then the estimator $\hat{f}_{n}$ obtained by
    \begin{equation*}
       \hat{f}_{n}\in\argmin_{f\in\cF_n}\sum_{i=1}^n\ell_{\mathsf{hinge}}(y_if(\x_i))
    \end{equation*}
satisfies
    \begin{equation*}
        \E\sbr{\cR_{01, \eta}(\hat{f}_n)-\min_{f\in\cF}\cR_{01, \eta}(f)}
        \le C_3\delta_n,
    \end{equation*}
for some universal constant $C_3>0$, where the expectation is taken over the training data.
\end{Lemma}

\begin{proof}[Proof of Theorem \ref{thm:cla}]
It is well known that $f^*=2\indi{\eta(\cdot)\ge1/2}-1=\argmin_{f\in\cF}\cR_{\mathsf{hinge}, \eta}(f)$, i.e., the hinge risk minimizer is equal to the Bayes classifier \citep{lin2004note}. The first step is to find a function $f_n\in\cF_{\sigma,n}$ which approximates the Bayes classifier $f^*$ well. Let $(\xi_n)_{n\in\bN}$ be a given sequence of positive integers. Since $\eta\in\cH^{\alpha, R}([0,1]^d)$, by Theorem \ref{sec:con}, for each $\xi_n$ there exists $\btheta_n$ such that  $\|N_\sigma(\cdot|\btheta_n)-\eta(\cdot)\|_\infty \le \xi_n$ with at most $O(\log(1/ \xi_n))$ layers, $O(\xi_n^{-d/\alpha})$ nodes at each layer and $O(\xi_n^{-d/\alpha}\log(1/ \xi_n))$ nonzero parameters. We construct the neural network $f_n$ by adding one ReLU layer to $N_\sigma(\cdot|\btheta_n)$ to have
    \begin{equation*}
        f_n(\x)=2\left\{\rho\left(\frac{1}{\xi_n}\left(N_\sigma(\x|\btheta_n)-\frac{1}{2}\right)\right)
-\rho\left(\frac{1}{\xi_n}\left(N_\sigma(\x|\btheta_n)-\frac{1}{2}\right)-1\right)\right\}-1,
    \end{equation*}
where $\rho$ is the ReLU activation function. 
Note that $f_n(\x)$ is equal to $1$ if $N_\sigma(\x|\btheta_n)\ge 1/2+\xi_n$, $(N_\sigma(\x|\btheta_n)-1/2)/\xi_n$ if $1/2\le(N_\sigma(\x|\btheta_n)< 1/2+\xi_n$ and $-1$ otherwise. Let
    \begin{equation*}
        B(4\xi_n)=\{\x:|2\eta(\x)-1|>4\xi_n\}.
    \end{equation*}
    
Then on $B(4\xi_n)$, $|f_n(\x)-f^*(\x)|=0$, since  $N_\sigma(\x|\btheta_n)-1/2=(\eta(\x)-1/2)-((N_\sigma(\x|\btheta_n)-\eta(\x))\ge\xi_n$ when $2\eta(\x)-1>4\xi_n.$ Similarly we can show that $N_\sigma(\x|\btheta_n)-1/2<-\xi_n$ when $2\eta(\x)-1<-4\xi_n$. Therefore  the Tsybakov noise condition (\ref{eq:noise}) implies 
   \begin{align*}
     \cR_{\mathsf{hinge}, \eta}(f_n)-\cR_{\mathsf{hinge}, \eta}(f^*) 
       &= \int|f_n(\x)-f^*(\x)||2\eta(\x)-1|\textup{d}\P_\x(\x) \\
       &=  \int_{B(4\xi_n)^c}|f_n(\x)-f^*(\x)||2\eta(\x)-1|\textup{d}\P_\x(\x)\\
       &\le 8\xi_n\Pr(|2\eta(\x)-1|\le 4\xi_n)\le C_1'\xi_n^{q+1},
   \end{align*}
for some $C_1'>0,$ where the first equality follows from Theorem 2.31 of \citep{steinwart2008support}. 

We take $\delta_n=C_1'\xi_n^{q+1}$. Then there are positive constants $L_0$, $N_0$, $S_0$ and $B_0$ such that $f_n\in\cF_{\sigma,n}$ where
    \begin{align*}
        \cF_{\sigma,n}:=\Big\{N_\sigma(\cdot|\btheta):&\norm{N_\sigma(\cdot|\btheta)}_\infty\le 1, \\
        &\btheta\in\Theta_{d,1}\del{L_0\log (\delta_n^{-1}), N_0\delta_n^{-d/\alpha(q+1)}, S_0\delta_n^{-d/\alpha(q+1)}\log(\delta_n^{-1}), B_0\delta_n^{-\kappa'}}\Big\},
    \end{align*}
for some $\kappa'>0$. Propostion \ref{prop:entropy} implies that the log covering number of $\cF_{\sigma, n}$ is bounded above by
    \begin{align*}
      \log\cN\del{\delta_n, \cF_{\sigma,n}, \|\cdot\|_\infty} \le \delta_n^{-d/\alpha(q+1)}\log^3(\delta_n^{-1}).
    \end{align*}
Note that to satisfy the entropy condition of Lemma \ref{lem:hinge}, $\delta_n$ should satisfy
    \begin{equation}
    \label{eq:delta}
        (\delta_n)^{\frac{d}{\alpha(q+1)}+\frac{q+2}{q+1}}\ge C_2'n^{-1}\log^3(\delta_n^{-1})
    \end{equation}
for some $C_2'>0$. If we let $\delta_n=(\log^3 n/n)^{\alpha(q+1)/(\alpha(q+2)+d)}$, 
the condition (\ref{eq:delta}) holds and so the proof is~done.
\end{proof}

\bibliographystyle{apalike}
\bibliography{reference-deep}
\end{document}